\documentclass{article}

\usepackage{PRIMEarxiv}

\usepackage[utf8]{inputenc} 
\usepackage[T1]{fontenc}    
\usepackage{hyperref}       
\usepackage{url}            
\usepackage{booktabs}       
\usepackage{amsfonts}       
\usepackage{nicefrac}       
\usepackage{microtype}      
\usepackage{lipsum}
\usepackage{fancyhdr}       
\usepackage{graphicx}       
\graphicspath{{media/}}     

\usepackage[round]{natbib}

\usepackage{amsmath,amsfonts,bm}
\usepackage{algorithm}
\usepackage{algpseudocode}

\def\eqref#1{equation~\ref{#1}}

\def\floor#1{\lfloor #1 \rfloor}
\def\1{\bm{1}}

\DeclareMathAlphabet{\mathsfit}{\encodingdefault}{\sfdefault}{m}{sl}
\SetMathAlphabet{\mathsfit}{bold}{\encodingdefault}{\sfdefault}{bx}{n}

\newcommand{\E}{\mathbb{E}}

\newcommand{\R}{\mathbb{R}}

\usepackage{hyperref}
\usepackage{url}

\usepackage{cleveref}
\usepackage{graphicx}
\usepackage{multirow}

\usepackage{amsmath}
\usepackage{amssymb}
\usepackage{mathtools}
\usepackage{amsthm}
\usepackage{siunitx}
\usepackage{cleveref}
\usepackage{algorithm}
\usepackage[noEnd=false,indLines=true,beginComment=//~]{algpseudocodex} 

\usepackage{caption}
\usepackage{subcaption}
\usepackage{booktabs}
\usepackage{float}

\theoremstyle{plain}
\newtheorem{theorem}{Theorem}[section]
\newtheorem{proposition}[theorem]{Proposition}
\newtheorem{lemma}[theorem]{Lemma}
\newtheorem{corollary}[theorem]{Corollary}
\theoremstyle{definition}

\theoremstyle{remark}
\newtheorem{remark}[theorem]{Remark}

\def\R{\mathbb{R}}
\def\E{\mathbb{E}}

\title{Spectrum Extraction and Clipping \\ for Implicitly Linear Layers
}

\author{
  Ali Ebrahimpour-Boroojeny \\
  Computer Science \\
  UIUC \\
  \texttt{ae20@illinois.edu} \\
   \And
  Matus Telgarsky \\
  Courant Institute \\
  NYU \\
  \texttt{mjt10041@nyu.edu} \\
   \And
  Hari Sundaram \\
  Computer Science \\
  UIUC \\
  \texttt{hs1@illinois.edu}
}

\begin{document}
\maketitle

\begin{abstract}
  We show the effectiveness of automatic differentiation in efficiently
  and correctly computing and controlling the spectrum of \emph{implicitly} linear operators, a rich family of layer types including all standard convolutional and dense layers. We provide the first clipping method which is correct for general convolution layers, and illuminate the representational limitation that caused correctness issues in prior work. We study the effect of the batch normalization layers when concatenated with convolutional layers and show how our clipping method can be applied to their composition. By comparing the accuracy and performance of our algorithms to the state-of-the-art methods, using various experiments, we show they are more precise and efficient and lead to better generalization and adversarial robustness. We provide
  the code for using our methods at \url{https://github.com/Ali-E/FastClip}. 
\end{abstract}

\keywords{Spectral Norm \and Convolutional Layers \and Regularization}

\section{INTRODUCTION}


 Implicitly linear layers are key components of deep learning models and include any layer whose output can be written as an affine function of their input. This affine function might be trivial, such as a dense layer, or non-trivial, such as convolutional layers. These layers inherit appealing properties of linear transformations; not only are they flexible and easy to train, but also they have the same Jacobian as the transformation that the layer represents and a Lipschitz constant equal to its largest singular value. Therefore controlling the largest singular value of these layers, which is the same as the largest singular value of their Jacobians, not only contributes to the generalization of the model~\citep{bartlett2017spectrally}, but makes the model more robust to adversarial perturbations~\citep{szegedy2013intriguing,weng2018evaluating}, and prevents the gradients from exploding or vanishing during backpropagation. Although efficient algorithms have been introduced to bound the spectral norm of dense layers~\citep{miyato2018spectral}, computing and bounding them efficiently and correctly has been a challenge for the general family of \textit{implicitly} linear layers, such as convolutional layers.

Convolutional layers are a major class of implicitly linear layers that are used in many models in various domains. They are compressed forms of linear transformations with an effective rank that depends on the dimensions of input rather than their filters. A 2d-convolutional layer with a stride of $1$ and zero padding, whose kernel has dimensions $(c_{out}, c_{in}, k, k)$ and is applied to an input of size $n\times n$ represents a linear transformation of rank $\min(c_{in}, c_{out}) n^2$. This compression is beneficial for training large and deep models as it reduces the number of parameters drastically while keeping the flexibility of performing the higher-rank transformation. Nevertheless, this compression makes it challenging to compute and control the spectrum. Clearly, a straightforward way of computing the spectrum of the convolutional layers is to unroll them to build the explicit matrix form of the transformation and compute its singular values. In addition to the complexity of computing the matrix representation for various padding types and sizes and different strides, this procedure is very slow, as represented in prior work~\citep{sedghi2018singular}, and forfeits the initial motivation for using them. The problem becomes much more challenging if one decides to bound the spectral norms of these layers during training without heavy computations that make this regularization impractical.

There has been an ongoing effort to design methods for controlling the spectrum of convolutional layers in either trained models or during the training
; however, neither of these methods work correctly for all standard convolutional layers. Moreover, they are either computationally extensive or rely on heuristics. 
In this work, we study both problems of extracting the spectrum and controlling them for the general family of implicitly linear layers. We give efficient algorithms for both tasks that scale to large models without incurring noticeable computational barriers. We also unveil some previously neglected limitations of convolutional layers in representing arbitrary spectrums. Up to this point, researchers have assumed they can modify the spectrum of convolutional layers in an arbitrary way. 
Although the experiments center around dense layers and convolutional layers, as we will discuss, our algorithms are general and do not assume properties that are specific to these layers. The only assumption is that their transformation can be represented as an affine function $f(x) = M_W x + b$, in which $W$ represents the parameters of the layer, but the transformation matrix $M_W$ is not necessarily explicit and might need further computations to be derived. 


In detail, the contributions are as follows:

\vspace{-3pt}

\textit{Efficient algorithm for extracting the spectrum:} We use auto-differentiation to implicitly perform shifted subspace iteration algorithm on any implicitly linear layer~(\Cref{sec-extraction}). Our method is correct for all convolutions and can be applied to the composition of layers as well. It is also more efficient than prior work for extracting top-$k$ singular values~(\Cref{sec-powerqr}).

\textit{Studying the limitations of convolutional layers:} We are the first to reveal the limitation of convolutional layers with circular padding in representing arbitrary spectrums~(\Cref{sec-limitations}). 

\textit{Fast and precise clipping algorithm:} We present a novel algorithm for exact clipping of the spectral norm to arbitrary values for implicitly linear layers in an iterative manner~(\Cref{sec-clipping}). Unlike prior methods, our algorithm works for any standard convolutional layer~(\Cref{fig:simple-clip-compare}) and can be applied to the composition of layers~(\Cref{sec-bn}). We show the effectiveness of our method to enhance the generalization and robustness of the trained models in\textbf{~\Cref{tab:bn-cifar}}.



\subsection{Related Work}

\citet{miyato2018spectral} approximate the original transformation by reshaping the kernel to a $n_{in}k^2 \times n_{out}$ matrix and perform the power-iteration they designed for the dense layers to compute the spectral norm of the convolutional layers. To clip the spectral norm, they simply scale all the parameters to get the desired value for the largest singular value.~\citet{farnia2018generalizable} and~\citet{gouk2021regularisation} use the transposed convolution layer to perform a similar power iteration method to the one introduced by~\citet{miyato2018spectral} for dense layers, and then perform the same scaling of the whole spectrum.~\citet{virmaux2018lipschitz} perform the power method implicitly using the automatic differentiation that is similar to our implicit approach; however, they do not provide any algorithm for clipping the spectral norms and leave that for future work. Also, for extracting multiple singular values, they use successive runs of their algorithm followed by deflation, which is much less efficient than our extraction method.
Some other works consider additional constraints.~\citet{cisse2017parseval} constrained the weight matrices to be orthonormal and performed the optimization on the manifold of orthogonal matrices.~\citet{sedghi2018singular} gave a solution that works only for convolutional layers with circular padding and no stride. For other types of convolutional layers, they approximate the spectral norm by considering the circular variants; however, the approximation error can be large for these methods.~\citep{senderovich2022towards} extends the method of~\citet{sedghi2018singular} to support convolutional layers with strides. It also provides an optional increase in speed and memory efficiency, albeit with the cost of expressiveness of convolutions, through specific compressions.~\citet{delattre2023efficient} use Gram iteration to derive an upper-bound on the spectral norm of the circular approximation in a more efficient way. However, they do not provide a solution for strides other than $1$. Therefore, in the best case, they will have the same approximation error as~\citep{sedghi2018singular}. A parallel line of research on controlling the spectrum of linear layers seeks to satisfy an additional property, gradient norm preserving, in addition to making each of the dense and convolutional layers $1$-Lipschitz~\cite{singla2021skew,yu2021constructing,trockman2021orthogonalizing,singla2021improved,xu2022lot,achour2022existence}. Some state-of-the-art methods in this line of work are also based on the clipping method introduced by~\citet{sedghi2018singular}~\citep{yu2021constructing,xu2022lot}. Therefore, those approaches are even slower and less efficient.
\section{METHODS}
\label{sec-methods}


This section is organized as follows. We introduce our algorithm, PowerQR, for extracting top-$k$ singular values and vectors of any implicitly linear layer or their composition in~\Cref{sec-extraction}. 
In~\Cref{sec-clipping}, we introduce our accurate and efficient algorithm, FastClip, for clipping the spectral norm of implicitly linear layers, and explain how it can be efficiently used with the PowerQR algorithm during training.  In~\Cref{sec-limitations}, we establish the inability of convolutions to attain any arbitrary spectrum.
The reader can find the proofs in~\Cref{apx-sec-omitted-proofs}. Next, we introduce notation.




 \paragraph{Notation.} An implicitly linear layer can be written as an affine function $f(x) = M_W x + b$, where $x \in \R^n$, $M_W \in \R^{m\times n}$, and $b \in \R^m$. $M_W$ is not necessarily the explicit form of the parameters of the layer, in which case the explicit form is shown with $W$ (e.g., the kernel of convolutional layers); otherwise, $M_W$ is the same as $W$. The spectral norm of $M_W$, which is the largest singular value of $M_W$, is shown with $\|M_W\|_2$. $\sigma_i(W)$ is used to represent the $i$-th largest singular value of matrix $M_W$. The largest singular value might be referred to as either $\|W\|_2$ or $\|M_W\|_2$. We use $\omega=\mathrm{exp}(2\pi i/n)$ (the basic $n$-th roots of unity), where $n$ is the rank of the linear transformation. $\Re(.)$ returns the real part of its input, and we also define the symmetric matrix $A_W$ as $M_W^\top M_W$. Hereinafter $[k]$ will be used to show the set $\{0,1,2,\dots, k\}$. When studying the composition of affine functions, we use the word “concatenation” to refer to the architecture of the network (i.e., succession of the layers), and “composition” to refer to the mathematical form of this concatenation. 

\subsection{Spectrum Extraction}
\label{sec-extraction}



We introduce our \textit{PowerQR} algorithm, which performs an implicit version of the shifted subspace iteration algorithm on any implicitly linear layer. Shifted subspace iteration is a common method for extracting the spectrum of linear transformations; the shift parameter makes the algorithm more stable and faster than the regular power method. Also, it is much more efficient when multiple singular values and vectors are of interest, compared to iterative calls to the regular power method followed by deflation (see chapter 5 of~\citet{saad2011numerical} for more details of this algorithm). However, the direct application of this algorithm requires the explicit matrix form of the transformation. In~\Cref{alg:powerqr}, we show how auto-differentiation can be used to perform this algorithm in an implicit way, without requiring the explicit matrix form. Regarding the complexity of~\Cref{alg:powerqr}, other than the $O(n^2 k)$ complexity of QR decomposition in line $6$, it has an additional cost of computing the gradients of the layer at line $4$ for a batch of $k$ data points, which is dominated by the former cost.



\begin{algorithm}[t]
\caption{PowerQR ($f, X, N, \mu=1$)}\label{alg:powerqr}
\begin{algorithmic}[1]
\State {\bfseries Input:} Affine function $f$, Initial matrix $X \in \R^{n\times k}$, Number of iterations $N$, Shift value $\mu$ 
\State {\bfseries Output:}  Top $k$ singular values and corresponding right singular vectors
    \For{$i=1$ {\bfseries to} $N$}
        \State $X^\prime \gets \nabla_X \frac{1}{2} \|f(X) - f(0)\|^2$ \\ \Comment{$\nabla_x \frac{1}{2} \|f(x) - f(0)\|^2 = M^\top Mx$}
        \State $X \gets \mu X + X^\prime$
        \State $(Q, R) \gets \mathrm{QR}(X)$ \Comment{$\mathrm{QR}$ decomposition of $X$}
        \State $X \gets Q$
    \EndFor

\State $S \gets \sqrt{\mathrm{diag}(R - \mu I)}$ \Comment{Singular values}
\State $V \gets X$  \Comment{Right singular vectors}
\State {\bfseries Return} $S,V$

\end{algorithmic}
\end{algorithm}



\begin{proposition}
\label{lemma-shifted-subspace}
Let $f(x) = Mx + b$. Then~\Cref{alg:powerqr} correctly performs the shifted subspace iteration algorithm on $M$, with $\mu$ as the shift value.   
\end{proposition}

 Note that by subtracting the value of $f$ at $0$ in line $4$, we remove the bias term from the affine function, which does not affect the singular values and vectors. This technique works even for the composition of multiple affine functions (e.g., concatenation of implicitly linear layers).
 The convergence rate in the subspace iteration algorithm depends on the initial matrix $X$ (line $1$), and if there are vectors $s_i$ in the column space of $X$ for which $\|s_i - v_i\|_2$ is small, then it takes fewer steps for the $i$-th eigenvector to converge. Meanwhile, we know that SGD makes small updates to $W$ in each training step. Therefore, by reusing the matrix $V$ in line $10$ (top-$k$ right singular vectors) computed for $W_{i-1}$ as initial $X$ for computing the spectrum of $W_i$, we can get much faster convergence. Our experiments show that by using this technique, it is enough to perform only one iteration of PowerQR per SGD step to converge to the spectrum of $W$ after a few steps. Using a warm start (performing more iterations of PowerQR for the initial $W$) helps to correctly follow the top-$k$ singular values during the training. The prior work on estimating and clipping the singular values has exploited these small SGD updates for the parameters in a similar manner~\citep{farnia2018generalizable,gouk2021regularisation,senderovich2022towards}.

\subsection{Clipping the Spectral Norm}
\label{sec-clipping}

In this section, we introduce our algorithm for clipping the spectral norm of an implicitly linear layer to a specific target value. To project a linear operator $M = USV^\top$ to the space of operators with a bounded spectral norm of $c$, it is enough to construct $A^\prime = US^\prime V^\top$, where $S^\prime_{i,i} = \min (S_{i,i}, c)$~\citep{lefkimmiatis2013hessian}. Therefore, for this projection, we will not need to extract and modify the whole spectrum of $A$ (as in~\citep{sedghi2018singular,senderovich2022towards}), and only clipping the singular values that are larger than $c$ to be exactly $c$ would be enough. Note that after computing the largest spectral norm (e.g., using the PowerQR method) by simply dividing the parameters of the affine model by the largest singular value, the desired bound on the spectral norm would be achieved, and that is the basis of some prior work~\citep{miyato2018spectral,farnia2018generalizable,gouk2021regularisation}; however, this procedure scales the whole spectrum rather than projecting it to a norm ball. Therefore, it might result in suboptimal optimization of the network due to replacing the layer with one in the norm ball that behaves very differently.
This adverse effect can be seen in~\Cref{tab:bn-cifar} as well. Thus, to resolve both aforementioned issues, our algorithm iteratively shrinks the largest singular value by substituting the corresponding rank $1$ subspace in an implicit manner. 


\begin{algorithm}[t]
\caption{Clip ($f_W, c,  N, P$)}\label{alg:clip}
\begin{algorithmic}[1]
\State {\bfseries Input:} Affine function $f_W= M_W x + b$, Clip value $c$, Number of iterations $N$, Learning rate $\lambda$, Number of iterations of PowerQR $P$
\State {\bfseries Output:}  Affine function $f_{W^\prime}$ with singular values clipped to $c$
\State $\sigma_1, v_1 \gets \mathrm{PowerQR}(f_{W}, v, P)$ \small{($v$: Random vector)}
    \While{$\sigma_1 > c$}
    \State $W^\prime \gets W$ \Comment{$W^\prime = \sum_{i=1}^{n} u_i \sigma_i v_i^\top$}
        \For{$i=1$ {\bfseries to} $N$}
            \State $W^\prime_\delta \gets \nabla_{W^\prime} \frac{1}{2} \|f_{W^\prime}(v_1) - f_{W}(c \sigma_1^{-1} v_1)\|^2$ \\ \Comment{ $W^\prime_\delta = u_1 (\sigma_1 - c)v_1^\top$}
            \State $W^\prime \gets W^\prime - \lambda W^\prime_\delta$ 
        \EndFor
        \State $W \gets W^\prime$
        \State $\sigma_1, v_1 \gets \mathrm{PowerQR}(f_{W}, v, P)$ \small{($v$: Random vector)} \Comment{Update $\sigma_1$ and $v_1$}
    \EndWhile

\State {\bfseries Return} $f_{W^\prime}, \sigma_1, v_1$

\end{algorithmic}
\end{algorithm}

\Cref{alg:clip} shows our stand-alone clipping method. The outer \texttt{while} loop clips the singular values of the linear operator one by one. After clipping each singular value, the PowerQR method in line \texttt{11} computes the new largest singular value and vector, and the next iteration of the loop performs the clipping on them. The clipping of a singular value is done by the \texttt{for} loop. Considering that in the \texttt{for} loop $M:= M_{W} = M_{W^\prime} = \sum_{i=1}^{n} u_i \sigma_i v_i^\top$, since $v_i$s are orthogonal and $v_i^\top v_i = 1$, for line \texttt{7} we have:
\vspace{-1 pt}
\begin{align*}
    W^\prime_\delta &= \nabla_{W^\prime} \frac{1}{2} \|f_{W^\prime}(v_1) - f_{W}(c \sigma_1^{-1} v_1) \|^2 \\
    &=  \left(f_{W^\prime}(v_1) - f_{W}(c \sigma_1^{-1} v_1)\right) \nabla_{W^\prime} f_{W^\prime}(v_1) \\
    &= \left(Mv_1 + b - c\sigma_1^{-1}Mv_1 - b \right) v_1^\top \\
    &= \left(u_1 \sigma_1 - u_1 c \right) v_1^\top = u_1 \sigma_1 v_1^\top - u_1 c v_1^\top,
\end{align*}
\noindent where $\nabla_{W^\prime} f_{W^\prime}(v_1) = v_1^\top$ because it is as if we are computing the gradient of the linear operator with respect to its transformation matrix. Therefore, in line \texttt{7} if $\lambda = 1$ we compute the new transformation matrix as $\sum_{i=1}^{n} u_i \sigma_i v_i^\top - u_1 \sigma_1 v_1^\top + u_1 c v_1^\top = u_1 c v_1^\top + \sum_{i=2}^{n} u_i \sigma_i v_i^\top$. Notice that $W^\prime_\delta$ is in the same format as the parameters of the linear operator (e.g., convolutional filter in convolutional layers). 

If we set $\lambda=1$, the largest singular value will be clipped to the desired value $c$, without requiring the \texttt{for} loop. That is indeed the case for dense layers. The reason that we let $\lambda$ be a parameter and use the \texttt{for} loop is that for convolutional layers, we noticed that a slightly lower value of $\lambda$ is required for the algorithm to work stably. The \texttt{for} loop allows this convergence to singular value $c$.




To derive our fast and precise clipping method, \textit{FastClip}, we intertwine~\Cref{alg:powerqr} and~\Cref{alg:clip} with the outer SGD iterations used for training the model, as shown in~\Cref{alg:fastclip}.
As we mentioned in~\Cref{sec-extraction}, the PowerQR method with warm start is able to track the largest singular values and corresponding vectors by running as few as one iteration per SGD step (lines $4$ and $8$ in~\Cref{alg:fastclip}). Whenever the clipping method is called, we use this value and its corresponding vector as additional inputs to~\Cref{alg:clip} (rather than line $3$ in~\Cref{alg:clip}). By performing this clipping every few iterations of SGD, since the number of calls to this method becomes large and the changes to the weight matrix are slow, we do not need to run its \texttt{while} loop many times. Our experiments showed performing the clipping method every $100$ steps and using only $1$ iteration of \texttt{while} and \texttt{for} loops is enough for clipping the trained models (lines $9$ and $10$). Because after the clipping, the corresponding singular value of $v$ has shrunk, we need to perform a few iterations of PowerQR on a new randomly chosen vector (since $v$ is orthogonal to the other right singular vectors) to find the new largest singular value and corresponding right singular vector, and line $11$ of~\Cref{alg:clip} takes care of this task. Obviously, the constants used in our algorithms are hyperparameters, but we did not tune them to find the best ones and the ones shown in this algorithm are what we used in all of our experiments. Also, any optimizer can be used for the SGD updates in line $7$ (e.g., Adam~\citep{kingma2014adam}).

\begin{algorithm}[tb]
\caption{FastClip ($f_W, X, c, N, \eta$)}\label{alg:fastclip}
\begin{algorithmic}[1]
\State {\bfseries Input:} Affine function $f_W$, Dataset $X$, Clip value $c$, Number of SGD iterations $N$, Learning rate $\eta$
\State {\bfseries Output:}  Trained affine function $f_{W^\prime}$ with 
singular values clipped to $c$
\State $v \gets \mathrm{Random\,input\,vector} $
\State $\sigma_1, v_1 \gets \mathrm{PowerQR} (f_W, v_1, 10)$
    \For{$i=1$ {\bfseries to} $N$}
        \State $X_b \gets \mathrm{SampleFrom} (X) $ \Comment{Sample a batch}
        \State $W = W - \eta \nabla_W \ell(f_W(X_b))$ \Comment{SGD step}
        \State $\sigma_1, v_1 \gets \mathrm{PowerQR} (f_W, v_1, 1)$
        \If{$ i \,\,\mathrm{isDivisibleBy} \,\,100$}
            \State $f_W, \sigma_1, v_1 \gets \mathrm{Clip} (f_W, \sigma_1, v_1, c, 1, 10)$
        \EndIf
    \EndFor

\State {\bfseries Return} $f_W$

\end{algorithmic}
\end{algorithm}

\subsection{Limitations of Convolutional Layers}
\label{sec-limitations}

In this section, we shed light on the formerly overlooked limitation of convolutional layers to represent any arbitrary spectrum. In some prior work, the proposed method for clipping computes the whole spectrum, clips the spectral norm, and then tries to form a new convolutional layer with the new spectrum~\citep{sedghi2018singular,senderovich2022towards}. We also introduce a new simple optimization method that uses our PowerQR algorithm and adheres to this procedure in~\Cref{sec-modification}. In the following, we start with a simple example that shows an issue with this procedure, and then, present our theoretical results that show a more general and fundamental limitation for a family of convolutional layers.

Consider a 2d convolutional layer whose kernel is $1\times 1$ with a value of $c$. Applying this kernel to any input of size $n\times n$ scales the values of the input by $c$. The equivalent matrix form of this linear transformation is an $n\times n$ identity matrix scaled by $c$. We know that this matrix has a rank of $n$, and all the singular values are equal to $1$. Therefore, if the new spectrum, $S^\prime$, does not represent a full-rank transformation, or if its singular values are not all equal, we cannot find a convolutional layer with a $1\times 1$ kernel with $S^\prime$ as its spectrum.


In the following theorem, we compute the closed form of the singular values of convolutional layers with circular padding, and in Remark~\ref{remark-duplicate}, we mention one of the general limitations that it entails.

\begin{theorem}
\label{theorem-limitations}
For a convolutional layer with $1$ input channel and $m$ output channels (the same result holds for convolutions with $1$ output channel and $m$ input channels) and circular padding applied to an input which in its vectorized form has a length of $n$, if the vectorized form of the $l$-th channel of the filter is given by $\textbf{f}^\textbf{(l)} = [f_0^{(l)}, f_1^{(l)}, \dots, f_{k-1}^{(l)}]$, the singular values are:
\vspace{-3 pt}
\begin{align}
    \mathcal{S}(\omega) = \left\{ \sqrt{\sum_{l=1}^m \mathcal{S}_j^{(l)^2}(\omega)}, \, j\in[n-1] \right\},
    \label{equ-eigval_2}
\end{align}
\vspace{-3pt}
\noindent where for $l \in {1,\dots, m}$:
\vspace{-3 pt}
\begin{align*}
    \mathcal{S}^{(l)}(\omega) = \left[\sqrt{c_0^{(l)} + 2\sum_i^{k-1} c_i^{(l)} \Re(\omega^{j\times i})},\, j\in [n-1] \right]^\top
    \label{equ-eigval}
\end{align*}
\vspace{-3pt}
\noindent in which $c_i^{(l)}$'s are defined as:
\begin{align*}
    c_0^{(l)} &:= f_0^{(l)^2} + f_1^{(l)^2} + \dots + f_{k-1}^{(l)^2},\\
    c_1^{(l)} &:= f_0^{(l)} f_1^{(l)} + f_1^{(l)} f_2^{(l)} + \dots + f_{k-2}^{(l)} f_{k-1}^{(l)},\\
    &\vdots\\
    c_{k-1}^{(l)} &:= f_0^{(l)} f_{k-1}^{(l)}.
\end{align*}

\label{pro-eigval_2}
\end{theorem}


\begin{remark}
\label{remark-duplicate}
 Note that in the closed form, we only have the real parts of the roots of unity. So, for any non-real root of unity, we get a duplicate singular value because mirroring that root around the real axis (i.e., flipping the sign of the imaginary part) gives another root of unity with the same real component. Only the roots that are real might derive singular values that do not have duplicates. The only such roots with real parts are $1$ and $-1$ ($-1$ is a root only for even $n$). Therefore, except for at most two singular values, other ones always have duplicates. This shows a more general limitation in the representation power of convolutional layers in representing arbitrary spectrums.
\end{remark}




As~\Cref{pro-eigval_2} shows the singular values of convolutional layers with circular padding have a specific structure. This makes the singular values of the convolutional layers connected. This structure among the singular values shrinks the space of the linear transformations they can represent. One such limitation was mentioned in Remark~\ref{remark-duplicate}. Using this theorem, we also derive an easy-to-compute upper bound and lower bound based on the values of the filter in Corollary~\ref{cor_eig_2}, which become equalities when the filter values are all non-negative. 

One implication of this result is that the alternative approach for clipping the spectral norm of convolutional layers that assigns arbitrary values to all the singular values at once (in contrast to our iterative updates which do not leave the space of convolutional operators), which has been used in several prior works, cannot be correct and effective because convolutional layers cannot represent arbitrary spectrums.

\begin{figure*}[t]
\centering
\includegraphics[width=1.\linewidth]{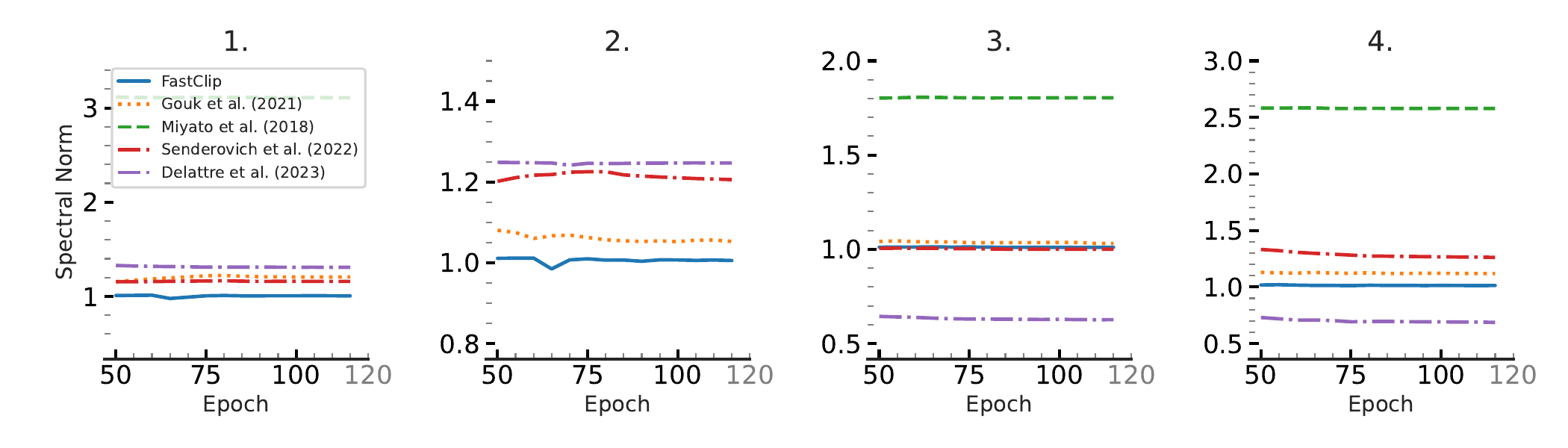}
\caption{Comparison of the clipping methods in a simple network with only one convolutional layer and one dense layer, where \textbf{the target value is $\pmb{1}$}. Our method is the only one that clips this layer correctly for all different settings: 1. Kernel of size $3$ with reflect padding, 2. Kernel of size $3$ with same padding, 3. Kernel of size $3$ and zeros padding with stride of $2$, and 4. Kernel of size $5$ with same replicate padding and stride of $2$.}
\label{fig:simple-clip-compare}
\end{figure*}

\subsection{Batch Normalization Layers}
\label{sec-bn}

Batch Normalization~\citep{ioffe2015batch} has proved to successfully stabilize and accelerate the training of deep neural networks and is thus by now standard in many architectures that contain convolutional layers. However, the adverse effect of these layers on the adversarial robustness of models has been noted in previous research~\cite{xie2019intriguing,benz2021revisiting,galloway2019batch}. As we will show in our experiments, not controlling the spectral norm of the batch normalization layers might forfeit the benefits of merely controlling the spectral norm of convolutional layers. We also point out an interesting compensation behavior that the batch normalization layer exhibits when the spectral norm of the convolutional layer is clipped; As the clipping value gets smaller, the spectral norm of the batch normalization layers increases (see~\Cref{fig-compensation}).

The clipping method introduced by~\citet{gouk2021regularisation} (explained in~\Cref{apx-alg-bn}), which is also used by the follow-up works~\citep{senderovich2022towards,delattre2023efficient}, performs the clipping of the batch norm layer separately from the preceding convolutional layer, and therefore upper-bounds the Lipschitz constant of their concatenation by the multiplication of their individual spectral norms. This upper-bound, although correct, is not tight, and the actual Lipschitz constant of the concatenation might be much smaller, which might lead to unwanted constraining of the concatenation such that it hurts the optimization of the model. Also, the purpose of the batch normalization layer is to control the behavior of its preceding convolutional layer, and therefore, clipping it separately does not seem to be the best option. As explained in~\Cref{sec-clipping}, our clipping method, can be applied directly to the concatenation of the convolutional layer and the batch normalization layer. This way, we can control the spectral norm of the concatenation without tweaking the batch normalization layer separately. We will show in our experiments that this method can be effective in increasing the robustness of the model, without compromising its accuracy.

\section{EXPERIMENTS}
\label{sec-experiments}

 
 We perform various experiments to show the effectiveness of our algorithms and compare them to the state-of-the-art methods to show their advantages. 
 Since the method introduced by~\citet{senderovich2022towards} extends the method of~\citet{sedghi2018singular} and~\citep{farnia2018generalizable} use the same method as~\citep{gouk2021regularisation}, the methods we use in our comparisons are the methods introduced by~\citet{miyato2018spectral,gouk2021regularisation,senderovich2022towards,delattre2023efficient}. For all these methods we used the settings recommended by the authors in all experiments. Implementations of our methods and experiments can be found in the supplementary material. All the experiments were performed on a single node with an NVIDIA A40 GPU. Some details of our experiments along with further results are presented in~\Cref{apx-exp}.

\subsection{PowerQR}
\label{sec-powerqr}

As pointed out previously, the methods introduced by~\citet{senderovich2022towards} and~\citet{delattre2023efficient}, compute the exact spectral norm only when circular padding is used for the convolutional layers. For the other types of common paddings for these layers, these methods can result in large errors. To show this, we computed the average absolute value of the difference in their computed values and the correct value for a $100$ randomly generated convolutional filters for different types of padding (see~\Cref{apx-exp-powerqr}). As the results show, the error dramatically increases as the number of channels increases. Also, the method introduced by~\citet{delattre2023efficient} assumes a stride of $1$ and therefore leads to even larger errors when a stride of $2$ is used. 
We also show that PowerQR is much more efficient than $k$ successive runs of the power method followed by deflation (as suggested in~\citet{virmaux2018lipschitz}) for extracting the top-$k$ singular values in~\Cref{apx-multi}. We also have utilized the capability of our method to extract the spectral norm of the concatenation of multiple implicitly linear layers for analyzing the spectral norm of the concatenation of convolutional and batch norm layers in~\Cref{fig:resnet18-spectral-norm} (see~\Cref{fig-resnet-catclip} for more use-cases).

\begin{figure*}[t]
\centering
\begin{subfigure}{.7\textwidth}
  \centering
  \includegraphics[width=.99\linewidth]{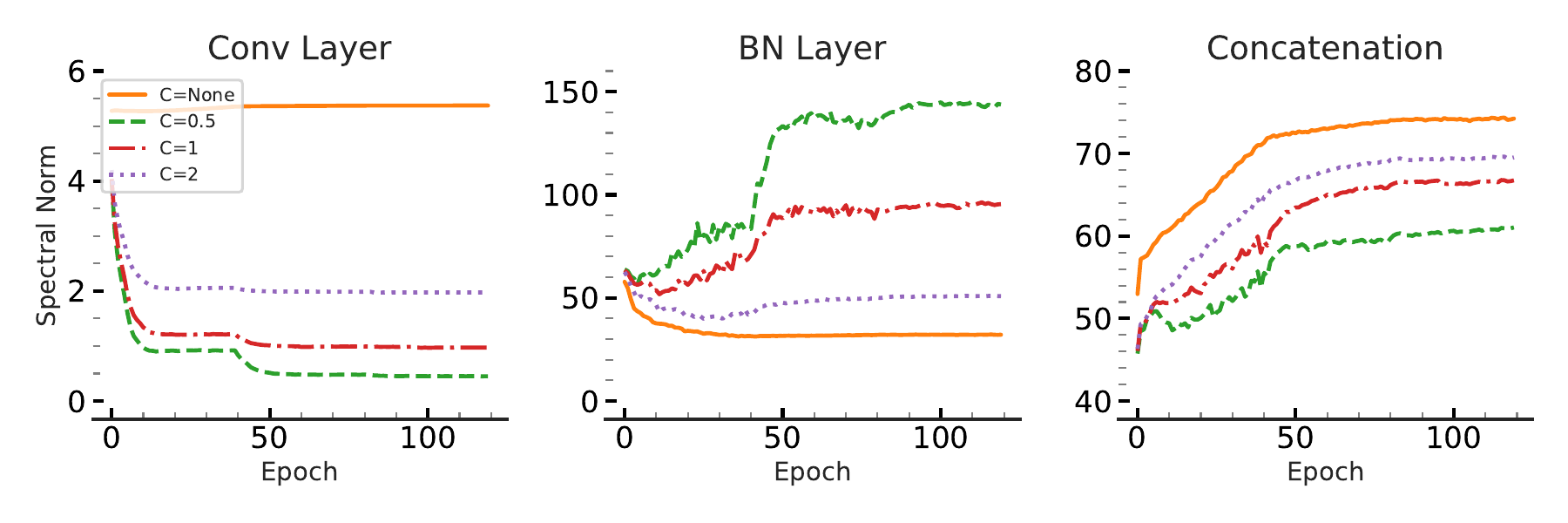}
  \caption{}
  \label{fig-compensation}
\end{subfigure}%
\hfill
\begin{subfigure}{.29\textwidth}
  \centering
  \includegraphics[width=.99\linewidth]{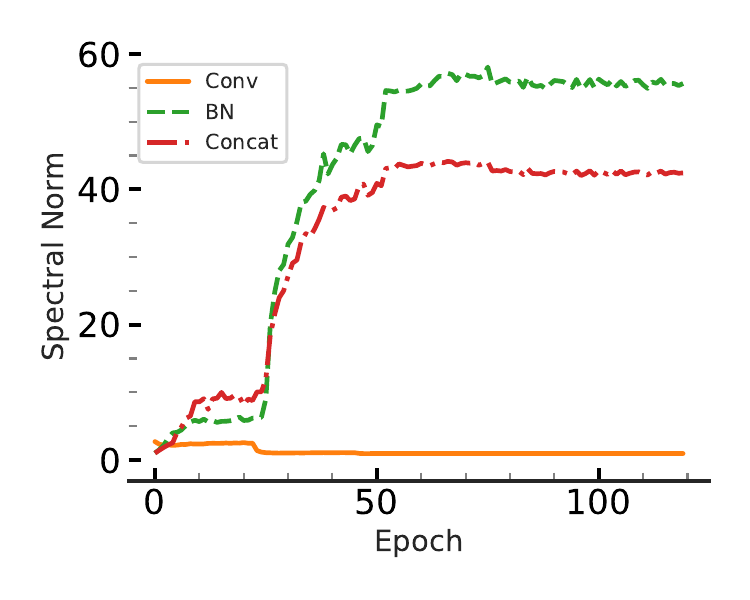}
  \caption{}
  \label{fig-resnet-concat}
\end{subfigure}
\caption{\textbf{(a)} The first three plots show the clipping of the convolutional layer in a simple two-layer network to various values on MNIST. As the clipping target gets smaller, the spectral norm of the batch norm layer compensates and becomes larger. Meanwhile, the spectral norm of their concatenation slightly decreases. \textbf{(b)} The right-most plot shows the spectral norm of a convolutional layer, its succeeding batch norm layer, and their concatenation from the clipped ResNet-18 model trained on CIFAR-10. Although the convolutional layer is clipped to $1$, the spectral norm of the concatenation is much larger due to the presence of the batch norm layer.}
\label{fig:resnet18-spectral-norm}
\end{figure*}

\subsection{Clipping Method}
\label{sec-experiment-clipping}

We start by comparing the correctness of the clipping models for different types of convolutional layers. Then, we show the effectiveness of each of these methods on the generalization and robustness of ResNet-18~\citep{he2016deep}, which is the same model used in the experiments of prior work~\citep{gouk2021regularisation,senderovich2022towards,delattre2023efficient} and Deep Layer Aggregation (DLA)~\citep{yu2018deep} model, which is more complex with more layers and parameters. We train the models on the same datasets used in prior works, MNIST~\citep{lecun1998mnist} and CIFAR-10~\citep{krizhevsky2009learning}.




\begin{figure}[t]
\centering
\includegraphics[width=1.\linewidth]{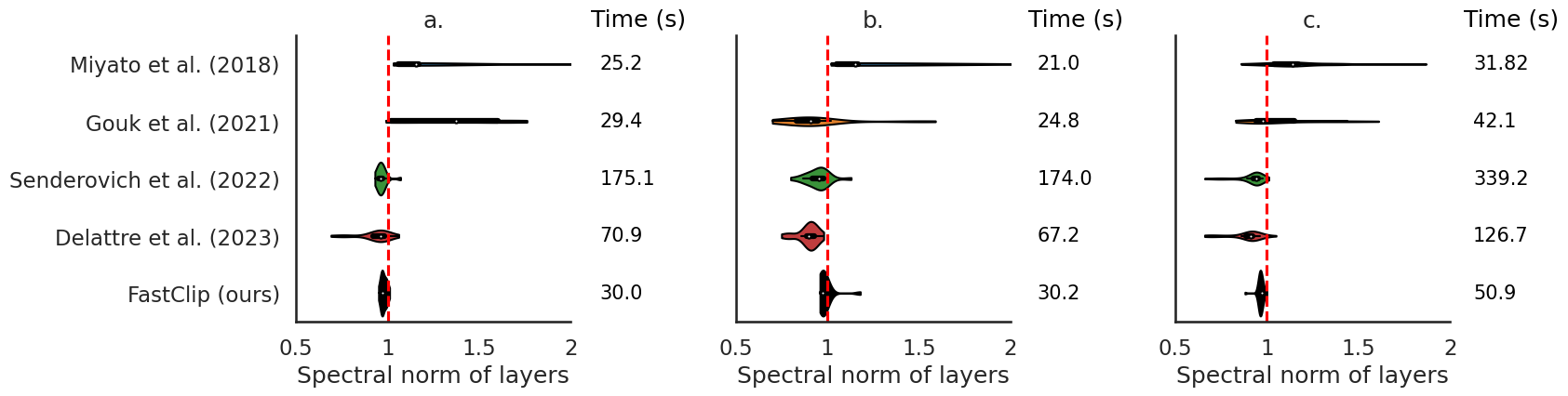}
\caption{The layer-wise spectral norm of a ResNet-18 model trained on CIFAR-10 (\textbf{a}) and MNIST (\textbf{b}) using each of the clipping methods. The time columns shows the training time per epoch for these methods. \textbf{c.} The layer-wise spectral norm of a DLA model trained on CIFAR-10 using each of the clipping methods. The time column shows the training time per epoch for these methods. As all of the plots show, by using our method, all the layers have a spectral norm very close to the \textbf{target value $\pmb{1}$}. Our method is also much faster than the relatively accurate alternatives and shows a slower increase in running time as the model gets larger.} 
\label{fig:resnet-clip-compare}
\end{figure}

\begin{table*}[t!]
\caption{The accuracy on the test set and adversarial examples generated with PGD-$(50)$ and CW for ResNet-18 trained on MNIST and CIFAR-10, for the models with their convolutional and dense layers clipped to $1$ (\texttt{With BN}) and clipped models with their batch norm layers removed (\texttt{No BN}). The accuracy of the original model on the test set, PGD-generated examples, and CW-generated examples for MNIST are $\pmb{99.37 \pm 0.02}$, $\pmb{15.30 \pm 9.11}$, and $\pmb{77.46 \pm 5.89}$, respectively. For CIFAR-10, these values are $\pmb{94.77 \pm 0.19}$, $\pmb{22.15 \pm 0.54}$, and $\pmb{13.85 \pm 0.74}$.}
\label{tab:bn-cifar}
\vskip 0.15in
\begin{center}
\begin{small}
\begin{sc}
\begin{tabular}{@{} l  c c  c  c  c c  c @{}}
 \toprule
 
 & \multicolumn{2}{@{}c}{\textbf{MNIST}} & \multicolumn{2}{@{}c}{\textbf{CIFAR-10}} \\\addlinespace[0.3em]

  \textbf{Method} &  With BN & No BN
 & With BN & No BN  \\\addlinespace[0.3em]

    \cmidrule(r){2-3}
    \cmidrule(r){4-5}

 & \multicolumn{4}{@{}c}{Accuracy on the test set}   \\ \addlinespace[0.4em]
 
 \cite{miyato2018spectral} &  $99.40 \pm 0.06$ & $98.00 \pm 0.22$ 
 & $94.82 \pm 0.11$ & $88.83 \pm 1.41$  \\\addlinespace[0.3em]

 \cite{gouk2021regularisation} & $99.25 \pm 0.04$ &  $21.94 \pm 6.01$ 
 & $89.98 \pm 0.38$ & $19.80 \pm 5.55$ \\\addlinespace[0.3em]

 \cite{senderovich2022towards} &  $99.40 \pm 0.03$ & $62.63 \pm 24.01$ 
 & $94.19 \pm 0.13$  &  $68.29 \pm 10.63$ \\\addlinespace[0.3em]

\cite{delattre2023efficient} &  $99.29 \pm 0.05$ & $97.27 \pm 0.03$ 
& $93.17 \pm 0.13$ &  $39.35 \pm 9.84$   \\\addlinespace[0.3em]


FastClip~(\Cref{alg:fastclip}) & $\pmb{99.41 \pm 0.04}$ & $\pmb{99.31 \pm 0.02}$      
& $\pmb{95.28 \pm 0.07}$ & $\pmb{92.08 \pm 0.28}$ \\\addlinespace[0.4em]


    & \multicolumn{4}{@{}c}{Accuracy on samples from PGD attack} 
        & \\\addlinespace[0.4em]
 
  \cite{miyato2018spectral} &  $21.77 \pm 12.98$ &  $32.67 \pm 14.08$  
  & $23.48 \pm 0.11$  & $35.18 \pm 7.72$  \\\addlinespace[0.3em]

  \cite{gouk2021regularisation} & $2.40 \pm 2.94$  & $8.41 \pm 3.03$  
  &  $16.13 \pm 1.28$  & $14.66 \pm 3.99$ \\\addlinespace[0.3em]

   \cite{senderovich2022towards} & $30.99 \pm 9.28$ & $15.97 \pm 4.84$ 
   &  $21.74 \pm 0.72$ & $39.84 \pm 7.87$  \\\addlinespace[0.3em]

\cite{delattre2023efficient} &  $30.87 \pm 4.77$  &  $71.75 \pm 1.49$ 
&  $21.08 \pm 0.84$ & $16.22 \pm 3.17$  \\\addlinespace[0.3em]


   FastClip~(\Cref{alg:fastclip}) &  $\pmb{47.90 \pm 5.49}$ & $\pmb{78.50 \pm 2.85}$  
    & $\pmb{24.48 \pm 0.32}$ & $\pmb{41.37 \pm 0.95}$ \\\addlinespace[0.4em]

     & \multicolumn{4}{@{}c}{Accuracy on samples from CW attack} 
         & \\\addlinespace[0.4em]
  
   \cite{miyato2018spectral} &  $86.25 \pm 2.18$ &  $73.56 \pm 10.38$  
   & $16.68 \pm 0.95$  & $48.48 \pm 6.40$  \\\addlinespace[0.3em]
 
   \cite{gouk2021regularisation} & $66.59 \pm 21.91$  & $21.94 \pm 6.01$  
   &  $18.79 \pm 2.99$  & $12.63 \pm 4.33$ \\\addlinespace[0.3em]
 
    \cite{senderovich2022towards} & $87.72 \pm 2.75$ & $58.71 \pm 20.67$ 
    &  $20.53 \pm 0.77$ & $43.82 \pm 9.57$  \\\addlinespace[0.3em]
 
 \cite{delattre2023efficient} &  $83.97 \pm 1.79$  &  $\pmb{96.93 \pm 0.06}$ 
 &  $24.05 \pm 1.72$ & $11.92 \pm 5.34$  \\\addlinespace[0.3em]
 
 
    FastClip~(\Cref{alg:fastclip}) &  $\pmb{90.21 \pm 1.80}$ & $95.35 \pm 1.06$  
     & $\pmb{24.31 \pm 0.96}$ & $\pmb{56.28 \pm 0.96}$ \\\addlinespace[0.4em]

\bottomrule
\end{tabular}
\end{sc}
\end{small}
\end{center}
\vskip -0.1in
\end{table*}

\subsubsection{Precision and Efficiency}
\label{subsec-comp-others}

We use a simple model with one convolutional layer and one dense layer and use each of the clipping methods on the convolutional layer while the model is being trained on MNIST and the target clipping value is $1$. We compute the true spectral norm after each epoch.~\Cref{fig:simple-clip-compare} shows the results of this experiment for $4$ convolutional layers with different settings (e.g., kernel size and padding type). This figure shows our method is the only one that correctly clips various convolutional layers.



In~\Cref{fig:resnet-clip-compare}a, we evaluated the distribution of the layer-wise spectral norms of the ResNet-18 models trained on CIFAR-10 using each of the clipping methods (for which the performance can also be found in~\Cref{tab:bn-cifar}) and showed our method is the most precise one, while being much more efficient than the relatively accurate alternatives~\citep{senderovich2022towards,delattre2023efficient}. To further evaluate the precision of our clipping method in comparison to the other methods, we also did the same experiment with the ResNet-18 model trained on the MNIST dataset. As~\Cref{fig:resnet-clip-compare}b shows the precision results are similar to the plot we had for CIFAR-10, with the times per epoch slightly smaller for all the methods due to the smaller input size. To check if the same precision results hold for other models and how the running time changes for larger models, we performed the same experiment for the DLA models that are trained using each of the clipping methods (for which the performance on the test set and adversarial examples can be found in~\Cref{tab:dla-results}). As~\Cref{fig:resnet-clip-compare}c shows, our method is still the most precise one among the clipping methods, while still being much faster than the more accurate alternatives. Another interesting point is that the factor by which the running times have increased for the larger model is smaller for our method compared to the methods introduced by~\citet{senderovich2022towards} and~\citet{delattre2023efficient}.

\subsubsection{Generalization and Robustness}
\label{sec-generalization}

Since the Lipschitz constant of a network may be upper-bounded by multiplying together the spectral norms of its constituent dense and convolutional layers, it is intuitive that regularizing per-layer spectral norms improves model generalization~\citep{bartlett2017spectrally}, and also adversarial robustness~\citep{szegedy2013intriguing}.
Therefore, we tested all the clipping models by using them during the training and computing the accuracy of the corresponding models on the test set and adversarial examples generated by two common adversarial attacks, Projected Gradient Descent (PGD)~\citep{madry2017towards} and Carlini \& Wanger Attack (CW)~\citep{carlini2017towards}.  

For the ResNet-18 model, we used the same variant used in prior work~\citep{gouk2021regularisation,senderovich2022towards,delattre2023efficient}. This variant divides the sum of the residual connection and convolutional output by $2$ before passing that as input to the next layer. This will make the whole layer $1$-Lipschitz if the convolutional layer is $1$-Lipschitz (the residual connection is $1$-Lipschitz). As~\Cref{tab:bn-cifar} shows, our clipping method leads to the best improvement in test accuracy while making the models more robust to adversarial attacks. The reason for the lack of the expected boost in the robustness of the models when clipping their spectral norms is shown in~\Cref{fig-resnet-concat}. This figure shows that although the models are clipped, the concatenation of some convolutional layers with batch normalization layers forms linear operators with large spectral norms. As~\Cref{fig-compensation} suggests, clipping the convolutional layer to smaller values will further increase the spectral norm of the batch norm layer. Still, as this figure suggests, there might be an overall decrease in the spectral norm of their concatenation which causes the slight improvement in the robustness of the clipped models. Therefore, we also trained a version of the model with all the batch norm layers removed. As~\Cref{tab:bn-cifar} shows, this leads to a huge improvement in the robustness of the clipped models; however, the clipped models achieve worse accuracy on the test set. 




\begin{figure*}[t]
\centering
\begin{subfigure}{.7\textwidth}
  \centering
  \includegraphics[width=.99\linewidth]{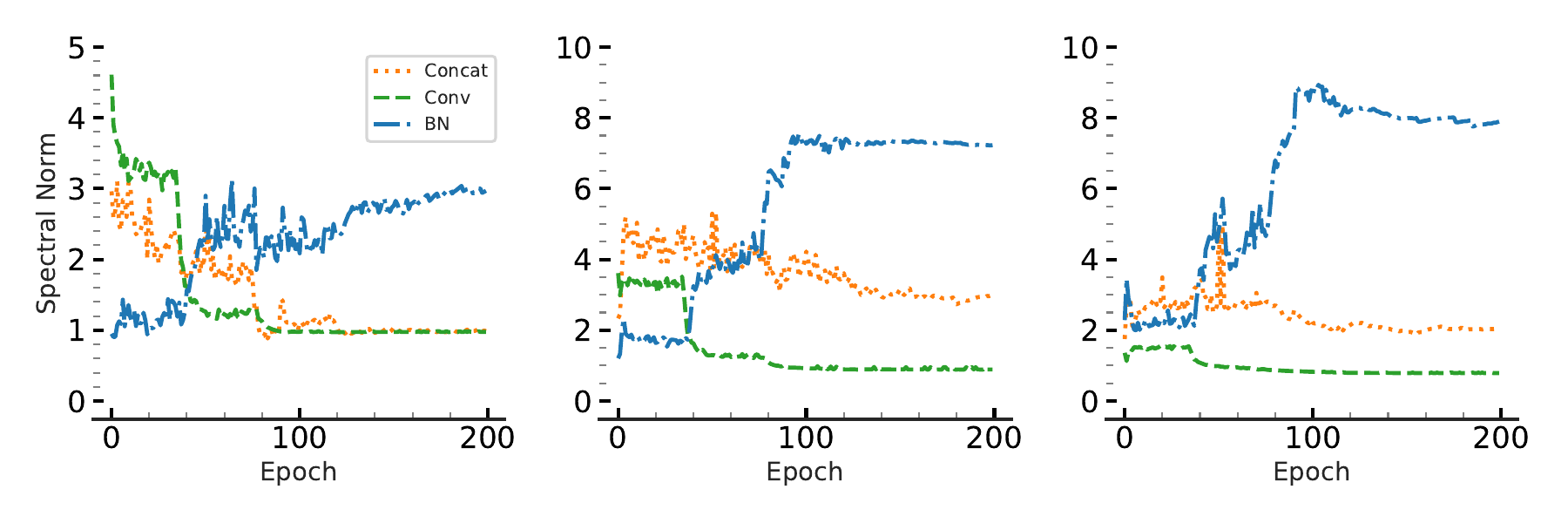}
  \caption{}
  \label{fig-resnet-catclip}
\end{subfigure}%
\hfill
\begin{subfigure}{.29\textwidth}
  \centering
  \includegraphics[width=.99\linewidth]{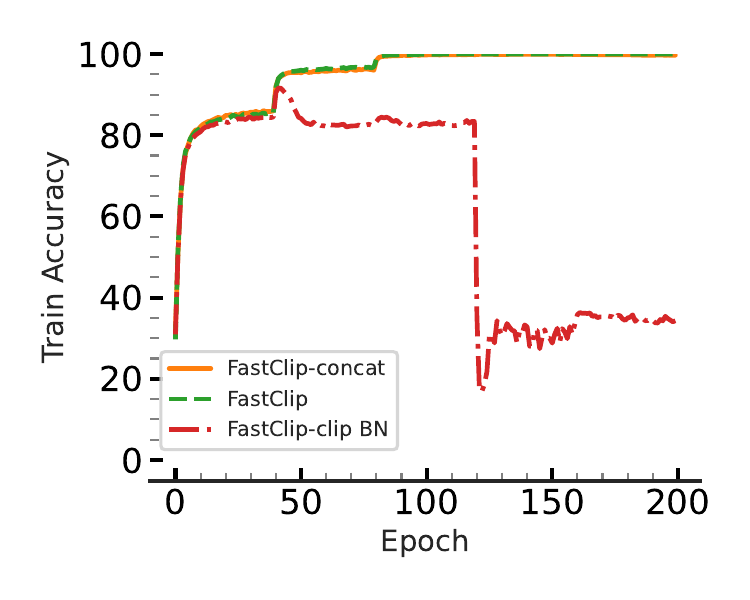}
  \caption{}
  \label{fig-resnet-bnclip-train-acc}
\end{subfigure}
\caption{\textbf{(a)} Each of these three subplots shows the spectral norms of a convolutional layer, its succeeding batch norm layer, and their concatenation in a ResNet-18 model trained on CIFAR-10. The convolutional layers in this model are clipped to $1$. Instead of clipping the batch normalization layer, our method has been applied to the concatenation to control its spectral norm. \textbf{(b)} The rightmost subplot shows the training accuracy for the ResNet-18 model that is trained on CIFAR-10. One curve belongs to the model with the convolutional layers clipped to $1$ using FastClip and the batch norm layers clipped using the direct method used by prior works (FastClip-clip BN). The other two belong to FastClip and FastClip-concat.}
\label{fig:resnet18-catclip-study}
\end{figure*}


We also performed the same experiment with the DLA model on both MNIST and CIFAR-10. These models are more complex than the ResNet-18 models and have more convolutional layers and parameters. An interesting observation with the DLA model was that they could not be trained at all without batch normalization layers with the default settings, even if the spectral norm is clipped using any of the clipping methods. Therefore, unlike the results for ResNet-18, we do not show the numbers for the \texttt{No BN} case (the test accuracies were almost the same as the random classifier in all cases). As~\Cref{tab:dla-results} shows, our clipping method makes the most improvement in generalization on both datasets and in terms of adversarial robustness, it is either better than the other methods or achieves competitive results. Note that, unlike ResNet-18 models in which according to the experiments by~\cite{gouk2021regularisation} we divided the sum of the two layers in the residual modules to make the whole output $1$-Lipschits at each layer, we used DLA models as they are designed without further modification. Since the DLA model contains modules in which the output of two layers are summed up to generate the inputs of the succeeding layer, making each one of the layers $1$-Lipshitz does not make each module $1$-Lipschitz. This might affect the robustness of models; however, we wanted to use this model as it is to show the benefits of our method applied to an original model without any modification.


\begin{table}[t]
\caption{The accuracy on the test set and adversarial examples generated with PGD-$(50)$ and CW for DLA trained on MNIST and CIFAR-10, for the models with their convolutional and dense layers clipped to $1$. The accuracy of the original model on the test set, PGD-generated examples, and CW-generated examples for MNIST are $\pmb{99.39 \pm 0.08}$, $\pmb{0.35 \pm 0.48}$, and $\pmb{50.78 \pm 6.58}$, respectively. For CIFAR-10, these values are $\pmb{94.82 \pm 0.12}$, $\pmb{21.67 \pm 1.42}$, and $\pmb{11.19 \pm 1.35}$.}
\label{tab:dla-results}
\vskip 0.15in
\begin{center}
\begin{small}
\begin{sc}
\begin{tabular}{@{} l  c c  c  c  c c  @{}}
 \toprule
 

\textbf{Method} &  \textbf{Test Set} & \textbf{PGD-(50)} & \textbf{CW}  \\\addlinespace[0.4em]


 & \multicolumn{3}{@{}c}{CIFAR-10}   \\ \addlinespace[0.2em]
\cmidrule(r){2-4}                        
 
 \cite{miyato2018spectral} &  $95.06 \pm 0.09$ & $\pmb{23.62 \pm 1.27}$ 
 & $\pmb{17.75 \pm 1.34}$  \\\addlinespace[0.3em]

 \cite{gouk2021regularisation} & $92.44 \pm 0.14$ &  $19.39 \pm 1.34$ 
 & $13.41 \pm 1.08$  \\\addlinespace[0.3em]

 \cite{senderovich2022towards} &  $93.39 \pm 2.30$ & $20.37 \pm 2.67$ 
 & $15.39 \pm 0.87$   \\\addlinespace[0.3em]

\cite{delattre2023efficient} &  $93.66 \pm 0.46$ & $19.90 \pm 2.47$ 
& $14.92 \pm 2.77$    \\\addlinespace[0.3em]

FastClip & $\pmb{95.53 \pm 0.10 }$ & $22.54 \pm 1.02$      
& $17.14 \pm 0.76$  \\\addlinespace[0.5em]


 & \multicolumn{3}{@{}c}{MNIST}   \\ \addlinespace[0.2em]

\cmidrule(r){2-4}                        

 \cite{miyato2018spectral} &  $99.40 \pm 0.05$ & $3.26 \pm 2.95$ 
 & $78.59 \pm 5.34$  \\\addlinespace[0.3em]

 \cite{gouk2021regularisation} & $99.26 \pm 0.06$ &  $3.13 \pm 2.49$ 
 & $79.16 \pm 7.07$  \\\addlinespace[0.3em]

 \cite{senderovich2022towards} &  $99.43 \pm 0.03$ & $14.03 \pm 5.27$ 
 & $83.92 \pm 2.33$   \\\addlinespace[0.3em]

\cite{delattre2023efficient} &  $99.34 \pm 0.04$ & $4.28 \pm 2.11$ 
& $68.02 \pm 5.37$    \\\addlinespace[0.3em]

FastClip & $\pmb{99.44 \pm 0.04}$ & $\pmb{16.74 \pm 7.09}$      
& $\pmb{84.90 \pm 1.59}$  \\\addlinespace[0.4em]

\bottomrule
\end{tabular}
\end{sc}
\end{small}
\end{center}
\vskip -0.1in
\end{table}

\subsection{Clipping Batch Norm}
\label{exp-bn-clip}


 As~\Cref{tab:bn-cifar} shows, the presence of batch normalization can be essential for achieving high test accuracy. In fact, as explained in~\Cref{sec-generalization}, DLA models cannot be trained without batch norm layers. So, instead of removing them, we are interested in a method that allows us to control their adverse effect on the robustness of the model. First, we explored the results for the models with their batch norm layers clipped to strictly less than $1$ by utilizing the method that was used by prior work~\citep{gouk2021regularisation,senderovich2022towards}. As the results show, our method still achieves the best results with this technique; however, this method for clipping the batch norm, although leads to bounded Lipschitz constant for the model, does not lead to a significant improvement in the robustness of the models and leads to low test accuracy, which is due to over-constraining the models and hindering their optimization as discussed in~\Cref{sec-bn}. In fact, as~\Cref{fig-resnet-bnclip-train-acc} shows, using this method for controlling batch norm layers even hinders the training process (we show this version of our method by \texttt{FastClip-clip BN} in our experiments). 

\begin{table}[t]
\caption{The accuracy on the test set and adversarial examples generated with PGD $(50), \epsilon=0.02$ of ResNet-18 trained on CIFAR-10, for the models with their convolutional and dense layers clipped to $1$ (\texttt{With BN}). For all the models, except \texttt{FastClip-concat}, the batch norm layers are clipped to strictly $1$ using the method by~\cite{gouk2021regularisation}. \texttt{FastClip-concat} uses FastClip method for controlling the batch norm of the concatenation of convolutional and batch norm layers, as described in~\Cref{exp-bn-clip}. As the results show, this method does not impede the optimization of the model and leads to a much better test accuracy while making the models more robust compared to when batch norm layers are not taken into account during the clipping process (see~\Cref{tab:bn-cifar} and \Cref{tab:dla-results}).}
\label{tab:bn-clip}
\vskip 0.15in
\begin{center}
\begin{small}
\begin{sc}
\begin{tabular}{@{} l  c c  c  c  c c  @{}}
 \toprule
 

\textbf{Method} &  \textbf{Test Set} & \textbf{PGD-50 $\epsilon=0.02$} & \textbf{CW $c=0.02$}  \\\addlinespace[0.4em]


 & \multicolumn{3}{@{}c}{ResNet-18 Model}   \\ \addlinespace[0.2em]
\cmidrule(r){2-4}                        
 
 \cite{miyato2018spectral} &  $85.91 \pm 0.27$ & $17.63 \pm 0.44$ 
 & $\pmb{49.94 \pm 2.25}$  \\\addlinespace[0.3em]

 \cite{gouk2021regularisation} & $27.33 \pm 2.11$ &  $13.89 \pm 0.68$ 
 & $11.23 \pm 3.53$  \\\addlinespace[0.3em]

 \cite{senderovich2022towards} &  $69.27 \pm 4.35$ & $16.22 \pm 1.86$ 
 & $25.27 \pm 3.43$   \\\addlinespace[0.3em]

\cite{delattre2023efficient} &  $30.44 \pm 3.59$ & $12.99 \pm 4.16$ 
& $10.17 \pm 6.63$    \\\addlinespace[0.3em]

FastClip & $90.59 \pm 0.36 $ & $\pmb{25.97 \pm 0.88}$      
& $41.50 \pm 2.02$  \\\addlinespace[0.6em]

FastClip-concat & $\pmb{94.63 \pm 0.08 }$ & $25.02 \pm 1.56$      
& $33.77 \pm 3.59$  \\\addlinespace[0.5em]


 & \multicolumn{3}{@{}c}{DLA Model}   \\ \addlinespace[0.2em]

\cmidrule(r){2-4}                        

 \cite{miyato2018spectral} &  $80.91 \pm 0.95$ & $16.29 \pm 2.79$ 
 & $\pmb{40.64 \pm 5.34}$  \\\addlinespace[0.3em]

 \cite{gouk2021regularisation} & $29.35 \pm 1.19$ &  $13.07 \pm 3.13$ 
 & $11.95 \pm 3.97$  \\\addlinespace[0.3em]

 \cite{senderovich2022towards} &  $74.94 \pm 1.03$ & $13.92 \pm 1.25$ 
 & $34.14 \pm 3.52$   \\\addlinespace[0.3em]

\cite{delattre2023efficient} &  $32.31 \pm 4.38$ & $12.18 \pm 3.14$ 
& $12.07 \pm 5.45$    \\\addlinespace[0.3em]

FastClip & $89.27 \pm 0.25$ & $23.40 \pm 1.46$      
& $39.16 \pm 2.76$  \\\addlinespace[0.6em]

FastClip-concat & $\pmb{95.02 \pm 0.07 }$ & $\pmb{25.93 \pm 1.31}$      
& $28.16 \pm 0.81$  \\\addlinespace[0.4em]
   
\bottomrule
\end{tabular}
\end{sc}
\end{small}
\end{center}
\vskip -0.1in
\end{table}

 The capability of our clipping method to be applied to the concatenation of implicitly linear layers provides an alternative approach to control the spectral norm of the concatenation of convolutional layers and batch norm layers. This still leads to the desired Lipschitz constants for the model, without over-constraining each individual layer. For this purpose, we clip the convolutional layers of the model to the target value, and meanwhile, we pass the modules that represent the composition of batch norm layers and their preceding convolutional layers to our clipping method while leaving the batch norm layers themselves unclipped. We will refer to this version of our clipping method as \texttt{FastClip-concat} in our experiments. In~\Cref{fig-resnet-catclip}, we show the effect of this approach on $3$ of the layers from the ResNet-18 model trained using \texttt{FastClip-concat}. As the plot shows, the convolutional layers are correctly clipped to $1$, and the spectral norm of the concatenations are approaching the target value $1$, while the spectral norm of the batch norm layers might increase up to an order of magnitude larger than the target clipping value. The convergence of the spectral norm of the concatenation to the target value is slower because we used a much smaller $\lambda$ value (see~\Cref{alg:clip}) to make the clipping method stable without changing the other hyperparameters.  
 
 ~\Cref{fig-resnet-bnclip-train-acc} shows the trajectory of the training accuracies for \texttt{FastClip}, \texttt{FastClip-concat}, and \texttt{FastClip-clip BN} (which uses \texttt{FastClip} together with the direct batch norm clipping method used in prior works~\cite{gouk2021regularisation}). As the figure shows, direct clipping of the batch norm layers hinders the optimization of the model and hence leads to poor results presented in~\Cref{tab:bn-clip}, while \texttt{FastClip-concat} follows the same trajectory as \texttt{FastClip} in terms of the training accuracy, which shows less interference with the optimization of the model. Next, we elaborate on the results presented in~\Cref{tab:bn-clip}.

We investigated the performance of both ResNet-18 and DLA models on the test set, as well as their adversarial robustness, when the clipping method for batch norm layers introduced by~\citet{gouk2021regularisation} is used. These results are presented in~\Cref{tab:bn-clip}. Moreover, we investigate the application of our clipping method to the concatenation of convolutional layers and their succeeding batch norm layers, rather than controlling the batch norm layers in isolation. We present the latter results, which is unique to our method, as \texttt{FastClip-concat} in~\Cref{tab:bn-clip}. In this setting, we use a smaller value for $\lambda$ (see~\Cref{alg:clip}), and apply the clipping every $500$ steps. As the results show, with the regular batch norm clipping introduced by~\citet{gouk2021regularisation} our method still achieves the best test accuracy among the models and increases the robustness compared to the original model, however, neither of the clipping methods can achieve generalization or adversarial robustness better than the \texttt{FastClip} model with the batch norm layers removed (see~\Cref{tab:bn-cifar} and~\Cref{tab:dla-results}). This shows the previously suggested clipping method for batch norm layers, although theoretically bounding the Lipschitz constant of the model, does not help us in practice. On the other hand, \texttt{FastClip-concat} helps us improve the robustness compared to the original model and \texttt{FastClip}, but still achieves an accuracy on the test set which is close to \texttt{FastClip} and much better than the models with their batch norm layers removed.

 As these results show, \texttt{FastClip-concat} is more successful in achieving its goal; it bounds the spectral norm of the concatenation and leads to improved robustness without incurring as much loss to the test accuracy compared to removing the batch normalization layers. Further optimization of the hyperparameters (e.g., clipping value of the convolutional layer, clipping value of the concatenation, $\lambda$, number of steps for clipping, etc.) for the convolutional layers and the concatenations, and finding the best combination is left for future work.

\section{CONCLUSIONS}

We introduced efficient and accurate algorithms for extracting and clipping the spectrum of implicitly linear layers. We showed they are more accurate and effective than existing methods. Also, our algorithms are unique in that they can be applied to not only convolutional and dense layers, but also the concatenation of these layers with other implicitly linear ones, such as batch normalization. This opens up the new possibilities in controlling the Lipschitz constant of models and needs to be explored further in future work. Through various experiments, we showed that our clipping method can be used as an effective regularization method for the neural networks containing convolutional and dense layers, which helps the model to generalize better on unseen samples and makes it more robust against adversarial attacks. 

\bibliography{example_paper}
\bibliographystyle{iclr2024_conference}

 \newpage

 \onecolumn
\appendix

\label{apx-A}

\section{Omitted Proofs}
\label{apx-sec-omitted-proofs}

\subsection{Proofs of Section~\ref{sec-extraction}}

We start by proving Proposition~\ref{lemma-shifted-subspace}:

\begin{proof}
 Let $g(X) = f(X) - f(0) = M_W X + b - b = M_W X$. Notice that $\nabla_X \frac{1}{2} \|g(X)\|^2 = M_W ^T M_W X = A_W X$. Let define $A^\prime = A_W + \mu I$. Putting steps $5$ and $6$ of the algorithm together, we have $(Q, R) = \mathrm{QR}(A_W X + \mu X) = \mathrm{QR}(A^\prime X)$. Therefore, the algorithm above can be seen as the same as the subspace iteration for matrix $A^\prime$. Hence, the diagonal values of $R$ will converge to the eigenvalues of $A^\prime$, and columns of $Q$ converge to the corresponding eigenvectors. If $\lambda$ is an eigenvalue of $A^\prime = A_W + \mu I$, then $\lambda - \mu$ is an eigenvalue of $A_W$. Therefore, by subtracting $\mu$ from the diagonal elements of $R$, we get the eigenvalues of $A_W$, which are squared of singular values of $M_W$. The eigenvectors of $A^\prime$ are the same as the eigenvectors of $A_W$ and, therefore, the same as the right singular vectors of $M_W$. 

\end{proof}

\subsection{Proofs of Section~\ref{sec-limitations}}

For proving~\Cref{theorem-limitations}, we first present and prove Lemma~\ref{pro-eigval} which proves a similar result for convolutional layers with only $1$ input and output channel.

\begin{lemma}
Given a convolutional layer with single input and output channel and circular padding applied to an input whose length of vectorized form is $n$, if the vectorized form of the kernel is given by $\textbf{f} = [f_0, f_1, \dots, f_{k-1}]$, the singular values are:
\begin{align}
    \mathcal{S}(\omega) = \left\{\sqrt{c_0 + 2\sum_i^{k-1} c_i \Re(\omega^{j\times i})},\,\, j=0,1,2, \dots, n-1 \right\}
\end{align}

\noindent where $c_i$'s are defined as:
\begin{align*}
    c_0 &:= f_0^2 + f_1^2 + \dots + f_{k-1}^2,\\
    c_1 &:= f_0f_1 + f_1f_2 + \dots + f_{k-2}f_{k-1},\\
    &\vdots\\
    c_{k-1} &:= f_0f_{k-1}.
\end{align*}

\label{pro-eigval}
\end{lemma}
\begin{proof}
The above convolutional operator is equivalent to a circulant matrix $A$ with $[f_m, f_{m+1}, \dots, f_{k-1}, 0, \dots, 0, f_0, f_1, \dots, f_{m-1}]^T$ as its first row, where
$m=\floor{\frac{k}{2}}$~\cite{jain1989fundamentals,sedghi2018singular}. Singular values of $A$ are the eigenvalues of $M=A^TA$. Circulant matrices are closed with respect to
multiplication, and therefore $M$ is a symmetric circulant matrix. It is easy to check that the first row of $M$ is $r = [c_0, c_1, \dots, c_{k-1}, 0, \dots, 0, c_{k-1},\dots, c_1]^T$. 

Now, we know the eigenvalues of a circulant matrix with first row $v = [a_0, a_1, \dots, a_{n-1}]^T$ are given by the set $\Lambda = \{ v\Omega_j, j=0,1,\dots,n-1 \}$, where $\Omega_j = [\omega^{j\times 0}, \omega^{j\times 1}, \dots, \omega^{j\times n-1}]$. So eigenvalues of $M$ are given by the set below:

\begin{align*}
    \{r\Omega_j,\, j=0,1,\dots, n-1\} &= \{c_0 + c_1\omega^{j\times 1} + c_2\omega^{j \times 2} + \dots + c_{k-1}\omega^{j\times (k-1)} \\
    &+ c_{k-1}\omega^{j\times (n-k+1)} + c_{k-2}\omega^{j\times (n-k+2)} + \dots + c_1\omega^{j\times (n-1)}, 
    \, j=0,1,\dots, n-1 \}.
\end{align*}

Note that $\omega^{j\times i}$ has the same real part as $\omega^{j\times (n-i)}$, but its imaginary part is mirrored with respect to the real axis (the sign is flipped). Therefore, $\omega^{j\times i} + \omega^{j\times (n-i)} = 2\Re(\omega^{j\times i})$. Using this equality the summation representing the eigenvalues of $M$ can be simplified to:

\begin{align*}
    \{c_0 + 2\sum_i^{k-1} c_i \Re(\omega^{j\times i}),\, j=0,1,\dots,n-1 \},
\end{align*}

and therefore the singular values can be derived by taking the square roots of these eigenvalues.
\end{proof}

Now using Lemma~\ref{pro-eigval}, we can easily prove~\Cref{theorem-limitations}.


\begin{proof}
The convolutional layer with $m$ output channels and one input channel can be represented by a matrix $M = [M_1, M_2, \dots, M_m]^T$, where each $M_i$ is an $n\times n$ circulant matrix representing the $i-$th channel. Therefore, $A = M^T M$ (or $MM^T$ when there are multiple input channels) can be written as $\sum_{l=1}^m M_l^T M_l$. So, for circulant matrix $A$ we have $c_i= \sum_{l=1}^m c_i^{(l)}$, and the proof can be completed by using Lemma~\ref{pro-eigval}. 
\end{proof}

Next, we focus on Corollary~\ref{cor_eig_2}, which uses the results of~\Cref{theorem-limitations} to give an easy-to-compute lower and upper bounds for the spectral norm of the convolutional layers with either one input or output channel. When the filter values are all positive, these bounds become equalities and provide and easy way to compute the exact spectral norm.

\begin{corollary}
Consider a convolutional layer with $1$ input channel and $m$ output channels or $1$ output channel and $m$ input channels and circular padding. If the vectorized form of the kernel of the $l$-th channel is given by $\textbf{f}^\textbf{(l)} = [f_0^{(l)}, f_1^{(l)}, \dots, f_{k-1}^{(l)}]$, and the largest singular value of the layer is $\sigma_1$, then:
\begin{align}
   \sqrt{\sum_{l=1}^{m} \left(\sum_{i=0}^{k-1} f_i^{(l)}\right)^2} \leq \sigma_1 \leq \sqrt{\sum_{l=1}^{m} \left(\sum_{i=0}^{k-1} |f_i^{(l)}| \right)^2},
\end{align}

\noindent and therefore the equalities hold if all $f_i^{(l)}$s are non-negative.

\label{cor_eig_2}
\end{corollary}

%
%

To make the proof more clear, we first present a similar corollary for Lemma~\ref{pro-eigval}, which proves similar results for convolutions with only $1$ input and output channel.

\begin{corollary}
Consider a convolutional layer with single input and output channels and circular padding. If the vectorized form of the kernel is given by $\textbf{f} = [f_0, f_1, \dots, f_{k-1}]$, and the largest singular value of the layer is $\sigma_1$, then:
\begin{align}
    \sum_{i=0}^{k-1} f_i \leq \sigma_1 \leq \sum_{i=0}^{k-1} |f_i|,
\end{align}

\noindent and therefore, the equalities hold if all $f_i$s are non-negative.

\label{cor_eig}
\end{corollary}

\begin{proof}
From Lemma~\ref{pro-eigval}, we can write:

\begin{align*}
    \mathcal{S}(\omega) &= \sqrt{c_0 + 2\sum_i^{k-1} c_i \Re(\omega^i)}
    = \sqrt{|c_0 + 2\sum_i^{k-1} c_i \Re(\omega^i)|}\\
    &\leq \sqrt{|c_0| + 2\sum_i^{k-1} |c_i \Re(\omega^i)|}
    \leq \sqrt{|c_0| + 2\sum_i^{k-1} |c_i|},
\end{align*}

\noindent where the last inequality is due to the fact that $\Re(\omega^i) \leq 1$. Now, for $c_i$s we have:
\begin{align*}
    |c_0| &= f_0^2 + f_1^2 + \dots + f_{k-1}^2,\\
    2|c_1| &\leq 2\left( |f_0||f_1| + |f_1||f_2| + \dots + |f_{k-2}||f_{k-1}|\right),\\
    &\vdots\\
    2|c_{k-1}| &\leq 2\left( |f_0||f_{k-1}|\right).
\end{align*}

Note that the summation of the right-hand sides of the above inequalities is equal to $(\sum_{i=0}^{k-1} |f_i|)^2$ Therefore:
\begin{align*}
    \mathcal{S}(\omega) \leq \sum_{i=0}^{k-1} |f_i| \implies \sigma_1 \leq \sum_{i=0}^{k-1} |f_i|.
\end{align*}

Since $1$ is always one of the $n-$th roots of unity, if all the $f_i$s are non-negative, all the above inequalities hold when $\omega = 1$ is considered.    
Now, note that when $\omega = 1$:

\begin{align*}
    \mathcal{S}^2(1) &= c_0 + 2\sum_i^{k-1} c_i \\
    &=f_0^2 + f_1^2 + \dots + f_{k-1}^2 \,\,\,\,\text{($c_0$)}\\
    &\quad + 2f_0f_1 + 2f_1f_2 + \dots + 2f_{k-2}f_{k-1}\,\,\,\,\text{($c_1$)}\\
    &\quad \,\,\,\,\vdots\\
    &\quad + 2f_0f_{k-1}\,\,\,\,\text{($c_{k-1}$)}\\
    &= (\sum_{i=0}^{k-1} f_i)^2
\end{align*}
Therefore, $\sum_{i=0}^{k-1} f_i$ is a singular value of the convolution layer, which gives a lower bound for the largest one and this completes the proof.

\end{proof}

Now we give the proof for Corollary~\ref{cor_eig_2}. 

%
%
%
%
%

\begin{proof}
    From~\Cref{theorem-limitations}, we can write:

\begin{align*}
    \mathcal{S}(\omega) = \sqrt{\sum_{l=1}^m \mathcal{S}^{(l)^2}(\omega)} \leq \sqrt{\sum_{l=1}^m (\sum_{i=0}^{k-1} |f_i^{(l)}|)^2},
\end{align*}

    \noindent where the inequality is a result of using Corollary~\ref{cor_eig} for each $\mathcal{S}^{(l)}(\omega)$. 
    For showing the left side of inequality, we set $\omega = 1$ ($1$ is one of the $n$-th roots of unity), and again use~\Cref{theorem-limitations} to get:
    
\begin{align*}
    \mathcal{S}(1) &= \sqrt{\sum_{l=1}^m \mathcal{S}^{(l)^2}(1)} = 
    \sqrt{c_0^{(l)} + 2\sum_i^{k-1} c_i^{(l)} \Re(1)}  
    = \sqrt{\sum_{l=1}^m (\sum_{i=0}^{k-1} f_i^{(l)})^2},
\end{align*}

\noindent where the last equality was shown in the proof of Corollary~\ref{cor_eig}. This shows the left side of the inequality in (4) is one of the singular values of the layer, and hence is a lower bound for the spectral norm (the largest singular value of the layer).
    
\end{proof}

\section{Other Algorithms}
\label{apx-B}

\subsection{Modifying the Whole Spectrum}
\label{sec-modification}

If we use PowerQR to extract singular values and right singular vectors $S$ and $V$ from $M_W = USV^T$, then
given new singular values $S^\prime$, we can modify the spectrum of our function to generate a linear operator $f^\prime: x \rightarrow M_W^\prime x + b$, where $M_W^\prime = US^\prime V^T$, without requiring the exact computation of $U$ or $M_W$:
\begin{align*}
    f_W(V S^{-1} S^\prime V^T x) - f(0) = USV^T V S^{-1} S^\prime V^T x 
    = US^\prime V^T x = f^\prime (x) - b.
\end{align*}

Although $f^\prime$ is a linear operator with the desired spectrum, it is not necessarily in the desired form. For example, if $f_W$ is a convolutional layer with $W$ as its kernel, $f_W(V S S^\prime V^T x)$ will not be in the form of a convolutional layer. To find a linear operator of the same form as $f_W$, we have to find new parameters $W^\prime$ that give us $f^\prime$. For this, we can form a convex objective function and use SGD to find the parameters $W^\prime$ via regression:
\begin{align}
    \min_{W^\prime} \mathbb{E}_x \left\|f_{W^\prime}(x) - f_W(V S^{-1} S^\prime V^T x)\right\|_\mathrm{F}^2,
    \label{equ_modify}
\end{align}

\noindent where $x$ is randomly sampled from the input domain. Note that we do not need many different vectors $x$ to be sampled for solving~\ref{equ_modify}. If the rank of the linear operator $f_W(.)$ is $n$, then sampling as few as $n$ random data points would be enough to find the optimizer $f_{W^\prime}$. This procedure can be seen as two successive projections, one to the space of linear operators with the desired spectrum and the other one to the space of operators with the same form as $f_W$ (e.g., convolutional layers). This method can be very slow because the matrix $V$ can be very large. One might reduce the computational cost by working with the top singular values and singular vectors and deriving a low-ranked operator. This presented method is not practical for controlling the spectral norm of large models during training, similar to the other algorithms that need the whole spectrum for controlling the spectral norm~\citep{sedghi2018singular,senderovich2022towards}.


Another important point to mention here is that depending on the class of linear operators we are working with, the objective~\ref{equ_modify} might not reach zero at its minimizer. For example, consider a 2d convolutional layer whose kernel is $1\times 1$ with a value of $c$. Applying this kernel to any input of size $n\times n$ scales the value of the input by the $c$. The equivalent matrix form of this linear transformation is an $n\times n$ identity matrix scaled by $c$. We know that this matrix has a rank of $n$, and all the singular values are equal to $1$. Therefore, if the new spectrum, $S^\prime$, does not represent a full-rank transformation, or if its singular values are not all equal, we cannot attain a value of $0$ in optimization~\ref{equ_modify}. As we showed in~\Cref{sec-limitations}, this problem is more general for the convolutional layers, as they are restricted in the spectrums they can represent.

\subsection{Clipping Batch Norm}
\label{apx-alg-bn}


Batch Normalization~\citep{ioffe2015batch} has proved to successfully stabilize and accelerate the training of deep neural networks and is thus by now standard in many architectures that contain convolutional layers. However, the adverse effect of these layers on the adversarial robustness of models has been noted in previous research~\cite{xie2019intriguing,benz2021revisiting,galloway2019batch}. As we showed in~\Cref{fig-resnet-concat}, not controlling the spectral norm of the batch normalization layers might forfeit the benefits of merely controlling the spectral norm of convolutional layers. We also revealed the compensating behavior of these layers as the spectral norm of convolutional layers are clipped to smaller values in~\Cref{fig-compensation}. The presented results in~\Cref{tab:bn-cifar} confirm this adverse effect of batch norm layers by showing a boost in the adversarial robustness of the models when these layers are removed from the network; however, as the results show, removing these layers from the model will incur a noticeable loss to the performance of the model on the test set and hinders the optimization. As pointed out in~\Cref{sec-generalization}, removing these layers from more complex modes such as DLA might completely hinder the optimization of the model. Therefore, it is crucial to find a better way to mitigate the adverse effect of these layers on adversarial robustness while benefiting from the presence of these layers in improving the optimization and generalization of the models. 

Batch normalization layers perform the following computation on the output of their preceding convolution layer:

\begin{align*}
    y = \frac{x-\E(x)}{\sqrt{\mathrm{Var}(x) + \epsilon}} \ast \gamma + \beta.
    \label{def-bn}
\end{align*}

As pointed out by~\cite{gouk2021regularisation}, by considering this layer as a linear transformation on $x-\E(x)$, the transformation matrix can be represented as a diagonal matrix with $\gamma_i/\sqrt{\mathrm{Var}(x_i)+ \epsilon}$ values, and therefore its largest singular value is equal to $\max_i \left(|\gamma_i|/\sqrt{\mathrm{Var}(x_i)+ \epsilon} \right)$. So, by changing the magnitude of $\gamma_i$ values we can clip the spectral norm of the batch norm layer. This approach has been followed in prior work~\citep{gouk2021regularisation,senderovich2022towards,delattre2023efficient}. In~\Cref{exp-bn-clip}, we show the results for the application of this clipping method and point out its disadvantages when used in practice. We will also show that the capability of our presented clipping method to work on the concatenation of convolutional and batch norm layers provides us with a better alternative. The full exploration of its potential, however, is left for future work.

\begin{figure}
\centering
\begin{subfigure}{.7\textwidth}
  \centering
  \includegraphics[width=.99\linewidth]{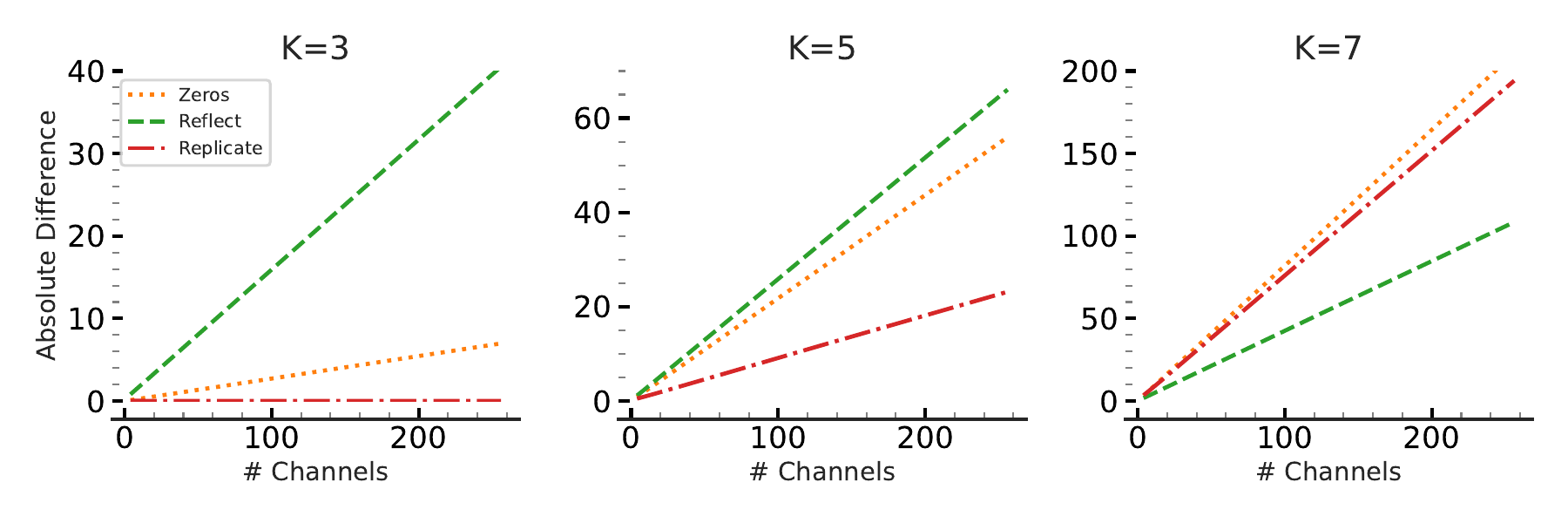}
  \caption{}
  \label{fig-circ-error}
\end{subfigure}%
\hfill
\begin{subfigure}{.3\textwidth}
  \centering
  \includegraphics[width=.95\linewidth]{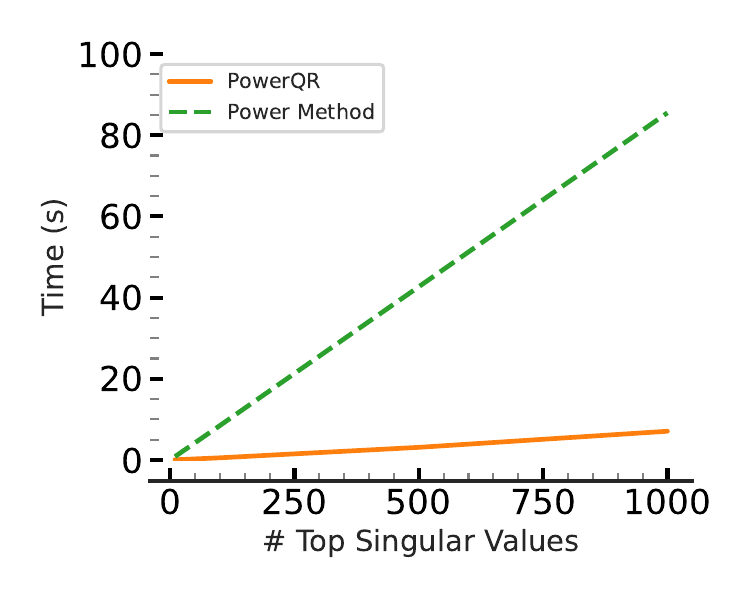}
  \caption{}
  \label{fig-QR-power-time}
\end{subfigure}
\caption{\textbf{a.} The absolute difference in the spectral norm of convolutional layers with different padding types and their circulant approximates for various kernel sizes ($3$, $5$, and $7$) and numbers of channels. The values are computed by averaging over $100$ convolutional filters drawn from a normal distribution for each setting. \textbf{b.} Comparison of the run-time of PowerQR (\cref{alg:powerqr}) to that of the pipeline used by~\citet{virmaux2018lipschitz} for computing the top-$k$ singular values. We considered a 2d-convolutional layer with $3\times 3$ filters and $32$ input/output channels. The convolution is applied to a $32\times 32$ image.}
\label{fig:fig-time-fft-small-and-fig-QR-power-time}
\end{figure}

\section{Experiments}
\label{apx-exp}

In this section, we elaborate on the experiments and results pointed out in section~\ref{sec-experiments} and present some additional experiments to show the advantages of our proposed methods. The models used for the experiments are ResNet-18, as it was used by prior work~\cite{gouk2021regularisation,senderovich2022towards,delattre2023efficient}, which has a small modification compared to the original model; for the residual connections in the model, they simply divide the output of the layer by $2$ so that if both of the layers are $1$-Lipschitz, the output of the module will remain $1$-Lipschitz. Similarly, for the optimization of these models, we use SGD optimizer and a simple scheduler that decays the learning rate by $0.1$ every $40$ steps. We train each of the models for $200$ epochs. In addition to ResNet-18 models, we perform the comparisons on DLA, as it is presented in a publicly available GitHub repository~\footnote[1]{https://github.com/kuangliu/pytorch-cifar} as a model that achieves better results on CIFAR-10. We use the same training procedure for these models as the ones used for the ResNet-18 model. For generating adversarial examples and evaluating the models, we use a publicly available repository~\footnote[2]{https://github.com/AI-secure/Transferability-Reduced-Smooth-Ensemble/tree/main}, along with the strongest default values for the attacks ($c = \epsilon=0.02$ and $c = \epsilon=0.1$ for CW and PGD attack on CIFAR-10 and MNIST, respectively), and also add their training procedure for MNIST to evaluate both ResNet-18 and DLA models on an additional dataset. We also use a simple model which consists of a single convolutional layer and a dense layer on the MNIST dataset for some of our experiments regarding the correctness of clipping methods~(\Cref{fig:simple-clip-compare}) and showing the compensation phenomenon~(\Cref{fig-compensation}).

\subsection{Error in Circulant Approximation}
\label{apx-exp-powerqr}

Our introduced method can be used to compute the exact spectrum of any implicitly linear layer, including different types of convolutional layers. However, the method introduced by~\citet{sedghi2018singular} only computes the spectrum of the convolutional layers when they have circular padding and no strides. This method was extended to convolutional layers with circular padding with strides other than $1$ by~\citet{senderovich2022towards}. These methods, as shown in~\Cref{subsec-comp-others}, are very slow and not practical for large models. These methods are also very memory consuming compared to other methods. More recently,~\citet{delattre2023efficient} introduced an algorithm that uses Gram iteration to derive a fast converging upper-bound for the convolutional layers with circular padding with a stride of $1$. As shown in~\Cref{fig:resnet-clip-compare}, this method is still much slower than our method for clipping and less accurate than our method and also its prior work on circulant convolutional layers. 
In~\Cref{fig-circ-error} we show that the approximation error using these method can be large on random convolutional layers. We use convolutional layers with different standard padding types with various filter sizes and choose the filter values from a normal distribution. Then we use the circulant approximation to approximate the spectral norm of these convolutional layers, and present the absolute difference with the true spectral norm. The presented values are averaged over $100$ trials for each setting. As the figure shows, the approximation error can be large, especially as the number of channels or the kernel size increases.

\subsection{Extracting Multiple Singular Values}
\label{apx-multi}

We also compared the run-time of the PowerQR with $100$ steps for extracting the top-$k$ eigenspace to the run-time of $k$ successive runs of the power method (as suggested in~\citet{virmaux2018lipschitz}). We did not include the deflation step (required for their proposed method to work), which makes the running time of their method even worse. Figure~\ref{fig-QR-power-time} in Appendix B shows that our implicit implementation of subspace iteration is dramatically more efficient.

\end{document}